\documentclass{article}
\usepackage{filecontents,lipsum}
\usepackage{amsthm}
\usepackage{comment}

\usepackage{bbm}
\newtheorem{remark}{Remark}
\newtheorem{lemma}{Lemma}

\newcommand{\dist}{{\rm dist}}
\usepackage[linesnumbered,ruled]{algorithm2e}
\usepackage{spconf,amsmath,graphicx}
\usepackage{lipsum}
\usepackage{graphicx}
 
\title{Training Generative Networks Using Random Discriminators}
\name{$^{\star}$Babak Barazandeh \qquad $^{\dagger}$Meisam Razaviyayn\qquad $^{\ddagger}$Maziar Sanjabi }

\address{ $^{\star}$$^{\dagger}$$^{\ddagger}$University of Southern California\\
\{barazand, razaviya,  sanjabi\}@usc.edu}

\usepackage{cite}

\usepackage{amsmath}
\usepackage{mathtools}
\usepackage{cite}

\usepackage{amsmath,amssymb,amsfonts}
\usepackage{algorithmic}
\usepackage{graphicx}
\usepackage{textcomp}
\usepackage{float}
\usepackage[utf8]{inputenc} 
\usepackage[T1]{fontenc}    
\usepackage{nth}
\usepackage{hyperref}       
\usepackage{url}            
\usepackage{booktabs}       
\usepackage{amsfonts}       
\usepackage{nicefrac}       
\usepackage{microtype}      
\usepackage{amsbsy} 
\usepackage{amsmath}

\newtheorem*{proof*}{Proof}

\def\BibTeX{{\rm B\kern-.05em{\sc i\kern-.025em b}\kern-.08em
		T\kern-.1667em\lower.7ex\hbox{E}\kern-.125emX}}

\usepackage{eso-pic}
\AddToShipoutPicture*{\footnotesize \sffamily\raisebox{1cm}{\hspace{1cm} \parbox{\textwidth}{Copyright 2019 IEEE. Published in the IEEE 2019 Data Science Workshop(DSW 2019), scheduled for June 2-5, 2019 in Minneapolis, Minnesota, USA. Personal use of this material is permitted. However, permission to reprint/republish this material for advertising or promotional purposes or for creating new collective works for resale or redistribution to servers or lists, or to reuse any copyrighted component of this work in other works, must be obtained from the IEEE. Contact: Manager, Copyrights and Permissions / IEEE Service Center / 445 Hoes Lane / P.O. Box 1331 / Piscataway, NJ 08855-1331, USA. Telephone: + Intl. 908-562-3966}}}
 
\begin{document}
		\setlength{\abovedisplayskip}{1.5pt}
	\setlength{\belowdisplayskip}{1.5pt}
%
\maketitle
\begin{abstract}
In recent years, Generative Adversarial Networks (GANs) have drawn a lot of attentions for learning the underlying distribution of  data in various applications. Despite their wide applicability, training GANs is notoriously difficult. This difficulty is due to the min-max nature of the resulting  optimization problem and the lack of proper tools of solving general (non-convex, non-concave) min-max optimization problems. In this paper, we try to alleviate this problem by proposing a new generative network that relies on the use of random discriminators instead of adversarial design. This design helps us to avoid the min-max formulation and leads to an optimization problem that is stable and could be solved efficiently. The performance of the proposed method is evaluated using handwritten digits (MNIST) and Fashion products (Fashion-MNIST) data sets. While the resulting images are not as sharp as adversarial training, the use of random discriminator  leads to a much faster algorithm as compared to the adversarial counterpart. This observation, at the minimum, illustrates the potential of the random discriminator approach for warm-start in training GANs. 
\end{abstract}
\section{Introduction}
Generative Adversarial Networks (GANs) \cite{goodfellow2014generative} have been relatively successful in learning underlying distribution of data, especially in application such as image generation. GANs aims to find the mapping that matches a known distribution to the underlying distribution of the data. The way they perform this task is by projecting the inputs to a higher dimension using Neural Networks \cite{sanjabi2018convergence} and then minimizing the distance between the mapped distribution and the unknown distribution in the projected space. To find the optimal network, \cite{goodfellow2014generative} proposed using Jensen-Shannon divergence\cite{lin1991divergence} for measuring the distance between projected distribution and the data distribution. Later on, \cite{nowozin2016f} generalized the idea by using the f-divergence as the measure. \cite{mao2017least} and \cite{zhao2016energy} proposed using least square and absolute deviation as the measure. 

The most recent works proposed using  Wasserstein distance and  Maximum Mean Discrepancy (MMD) as the distance measure\cite{arjovsky2017wasserstein,gulrajani2017improved, binkowski2018demystifying}. Unlike Jensen-Shannon divergence, the recent measures are continuous and almost everywhere differentiable. The common thread between all these approaches is that the problem is usually formulated as a game between two agents, i.e. generator and discriminator. Generator's role is to generate samples as close as possible to real data and discriminator is responsible for distinguishing between real data and the generated samples. The result is a non-convex min-max game which is difficult to solve. The learning process, which should solve the resulting non-convex min-max game, is hard to tackle, due to many factors such as using discontinuous \cite{arjovsky2017wasserstein} or non-smooth \cite{sanjabi2018convergence} measure. In addition to these factors, the fact that all of these models try to learn the mapping transformation adversarially makes the training unstable. 
Adding regularization or starting from a good initial point is one approach to overcome these problems \cite{sanjabi2018convergence}. However, for most problems finding a good initial point might be as hard as solving the problem itself.

Randomization has shown promising improvement in machine learning algorithms \cite{barazandeh2018behavior,sun2018random}. As the result, to prevent over-mentioned issues,  we propose learning underlying distribution of data not through adversarial player but through a random projection. This random projection not only decreases the computation time by removing the optimization steps needed for most of the discriminator's role, but also leads to a more stable optimization problem. The proposed method has the state of the art performance for simple datasets such as MNIST and Fashion-MNIST. 

\section{Problem Formulation}
Let $x\in \mathbb{R}^d$ be a random variable with distribution $P_x$ representing the real data; and $z$ be a random variable representing a known distribution such as standard Gaussian. Our goal is to find a function or a neural network $G(\cdot)$ such  that $G(z)$ has a similar distribution to the real data distribution $P_x$. Therefore,  our  objective is to solve the following optimization problem
\begin{equation}
\label{eq:OriginalProblems}
\min_{G} \;\; \dist(P_{G(z)}, P_x),
\end{equation}
where $P_{G(z)}$ is the distribution of $G(z)$ and $\dist(\cdot,\cdot)$ is a distance measure between the two distributions.

A natural question to ask is about what distance metric to use. The original paper of Goodfellow \cite{goodfellow2014generative} suggests the use of Jensen--Shannon divergence. However, as mentioned in \cite{arjovsky2017wasserstein}, this divergence is not continuous. Therefore, \cite{arjovsky2017wasserstein,sanjabi2018convergence} suggest to use the optimal transport distance. In what follows, we first review this distance and then discuss our methodology for solving~\eqref{eq:OriginalProblems}.

\section{Optimal Transport Distance }
Let $p$ and $q$ be two discrete distributions  taking $m$ different values/states. Thus   the distributions $p$ and $q$ can be represented by $m$-dimensional vectors  $(p_{1},\ldots, p_m)$ and $(q_1,\ldots,q_m)$ . The optimal transport distance is defined as the minimum amount of work needs to be done for transporting distribution $p$ to $q$ (and vice versa). Let $\pi_{i,j}$ be the amount of mass moved from state $i$ to state $j$; and $c_{ij}$ represent the per-unit cost of  this move. 
Then the optimal transport distance  between the two distributions $p$ and $q$ is defined as \cite{villani2009optimal}:
\begin{equation}
\label{eq:EarthMoverPrimal}
\begin{split}
\textrm{dist}(p, q) = \min_{\pi \geq 0} \quad  &\sum_{i=1}^m\sum_{j=1}^m c_{ij} \pi_{ij}\\
\textrm{s.t.} \quad & \sum_{j=1}^m \pi_{ij} = p_i,\; \forall i = 1,\ldots,m\\
&\sum_{i=1}^m \pi_{ij} = q_j, \;\forall j = 1,\ldots,m,
\end{split}
\end{equation}
where the constrains guarantee that the mapping $\pi$ is a valid transport.  In practice, a popular approach is to solve the dual problem. It is not hard to see that the dual of the  optimization problem \eqref{eq:EarthMoverPrimal} can be written as
\begin{equation}
\label{eq:EarthMoverDual}
\begin{split}
\textrm{dist}(p, q) = \max_{\lambda,\gamma} \quad  &\sum_{i=1}^m\gamma_i p_i + \sum_{j=1}^m \lambda_j q_j\\
\textrm{s.t.} \quad & \lambda_j + \gamma_i \leq c_{ij},\;\;\forall i,j=1,\ldots, m.
\end{split}
\end{equation}

When $c$ is a proper distance, this dual variable should satisfy $\lambda = -\gamma$ \cite{villani2009optimal}. In practice, since the dimension $m$ is large and estimating $p$ and $q$ accurately is not possible, we parameterize the dual  variable with a neural network and solve the dual optimization problem by training two  neural networks simultaneously \cite{arjovsky2017wasserstein}. However, this approach leads to a non-convex min-max optimization problem. Unlike special cases such as convex-concave set-up \cite{juditsky2016solving},  there is no algorithm to date in the literature which can find even an $\epsilon$-stationary point in the general non-convex setting; see \cite{ nouiehed2019solving } and the references therein. Therefore, training generative adversarial networks (GANs) can become notoriously  difficult in practice and may require significant tuning of training parameters. A natural solution is to not parameterize the dual function and instead solve \eqref{eq:EarthMoverPrimal} or \eqref{eq:EarthMoverDual} directly which leads to a convex reformulation. However, as mentioned earlier, since the dimension $m$ is large, approximating $p$ and $q$  is statistically not possible. Moreover, the distance in the original feature domain may not reflect the actual distance between the distributions. Thus, we suggest an alternative formulation in the next section.
 

\section{Training in different feature domain}
In many applications, the closeness of samples in the original  feature domain does not reflect the actual similarity between the samples. For example, two images of the same object may have a large difference when the distance is computed in the pixel domain. Therefore, other mappings of the features, such as features obtained by  Convolutional Neural Network (CNN) may be used to extract meaningful features from samples \cite{o2015introduction}.  

Let $\mathcal{D} = \{D_1,D_2,\ldots,D_K\}$  be a collection of meaningful features we are interested in. In other words, each function $D\in \mathcal{D} $ is a mapping from our original feature domain to the domain of interest, i.e.,  $D_k(\cdot):\mathbb{R}^d \mapsto \mathbb{R}^{d'}, \forall k=1,\ldots,K$. Then, instead of solving~\eqref{eq:OriginalProblems}, one might be interested in solving the following optimization problem
\begin{equation}
\label{eq:DistwithD}
\min_{G} \;\; \sum_{k=1}^K w_k\dist(P_{D_k(G(z))}, P_{D_k(x)}),
\end{equation}
where $P_{D_k(G(z))}$ represents the distribution of the random variable $D_k(G(z))$;  $P_{D_k(x)}$ is the distribution of $D_k(x)$; and $w_k$ is a weight coefficient indicating the importance of the $k$-th feature~$D_k$. 

In the general setting, we may have uncountable number of mappings $D_k$. Thus, by defining a measure on the set $\mathcal{D}$, we can generalize \eqref{eq:DistwithD} to the following optimization problem

\begin{equation}
\label{eq:DistwithDExp}
\min_{G} \;\; \mathbb{E}_D \bigg[\dist\left(P_{D(G(z))}, P_{D(x)}\right)\bigg].
\end{equation}

\begin{remark}
We use the notation $D$  since the function $D$ plays the role of a \textit{discriminator} in the Generative Adversarial Learning (GANs) context. 
\end{remark} 

Plugging ~\eqref{eq:EarthMoverPrimal} in the equation \eqref{eq:DistwithDExp} leads to the optimization problem

\begin{equation}
\label{eq:DistwithDExpExpanded}
\min_{G} \;\; \mathbb{E}_D \left[
\begin{array}{ll}
 \displaystyle{\max_{(\lambda,\gamma) \in \mathcal{C} }} \;  &\displaystyle{\sum_{i=1}^m}\gamma_i P_{D(G(z))}^i + \displaystyle{\sum_{j=1}^m} \lambda_j P_{D(x)}^j\\
 \vspace{0.2cm}\\
\;\;\;\textrm{s.t.} \quad & \lambda_j + \gamma_i \leq c_{ij},\;\;\forall i,j.
\end{array}
\right],
\end{equation}
where $\mathcal{C} = \{(\lambda, \gamma) | \;  \lambda_{i} + \gamma_{j} \leq c_{i,j}, \; \forall i,j\}$.\\

Unfortunately, in practice, we do not have access to the actual values of the distributions $P_{D(x)}$ and $P_{D(G(z))}$. However, we can estimate them using a batch of generated and real samples. The following simple lemma motivates the use of a natural surrogate function. 

\begin{lemma}
\label{lem:upperbound}
Let $p$ and $q$ be two discrete distributions with $p = (p_1,\ldots,p_m)$ and $q= (q_1,\ldots,q_m)$. Let $x\in \mathbb{R}^m$ and $y\in \mathbb{R}^m$ be the corresponding one-hot encoded random variables, i.e., $P(x = e_i) = p_i,\forall i=1,\ldots,m$ and $P(y = e_i) = q_i,\forall i=1,\ldots,m$, where $e_i$ is the $i$-th standard basis. Assume further that $\dist(p,q)$ is the optimal transport distance between $p$ and $q$ defined in \eqref{eq:EarthMoverPrimal}. Let $\hat{p}^n $ and $\hat{q}^n$ be the natural unbiased estimator of $p$ and $q$ based on $n$ i.i.d. samples. In other words, $\hat{p}^n = \frac{1}{n}\sum_{\ell=1}^n x_\ell$ and $\hat{q}^n = \frac{1}{n}\sum_{\ell=1}^n y_\ell$, where $x_\ell$ and $ y_\ell, \ell = 1,\ldots, n,$ are i.i.d samples obtained from distributions $p$ and $q$, respectively. Then,\\
\[
\mathbb{E}\left[\dist(\hat{p}^{n+1},\hat{q}^{n+1})\right] \leq \mathbb{E}\left[\dist(\hat{p}^n,\hat{q}^n)\right].
\]
Moreover, 
\[
\lim_{n\rightarrow \infty} \dist(\hat{p}^{n},\hat{q}^{n}) =  \dist({p},{q}), \;almost \;surely.
\]
\end{lemma}
\begin{proof}
The proof is similar to the standard proof in sample average approximation method; see \cite[Proposition 5.6]{shapiro2009lectures}.  Notice that,
\begin{equation}
\begin{aligned}
&\mathbb{E} \left[\textrm{dist}(\hat{p}^{n+1},\hat{q}^{n+1})\right] \\
= \;& \mathbb{E} \left[\displaystyle{\max_{(\lambda,\gamma) \in \cal{C}}} \;  
\displaystyle{\sum_{i=1}^m}\gamma_i \hat{p}^{n+1}_i + \displaystyle{\sum_{j=1}^m} \lambda_j \hat{q}^{n+1}_j\right]  \\ 
= \; &\mathbb{E} \left[\displaystyle{\max_{(\lambda,\gamma) \in \cal{C}}} \;  \langle {\hat{p}^{n+1}},{\gamma} \rangle + \langle {\hat{q}^{n+1}},{\lambda}\rangle\right] \\
=\;& \frac{1}{n+1}\mathbb{E} \left[\displaystyle{\max_{(\lambda,\gamma) \in \cal{C}}} 
 \langle\gamma, \sum_{\ell} x_{\ell} \rangle +  \langle\lambda, \sum_{\ell} y_{\ell} \rangle \right] \\
=\;& \frac{1}{n(n+1)} \mathbb{E} \left[\displaystyle{\max_{(\lambda,\gamma) \in \cal{C}}} \sum_{t = 1}^{n +1 } \langle\gamma, \sum_{\ell \neq t}x_{\ell}\rangle 
 +  \sum_{t = 1}^{n +1 } \langle \lambda, \sum_{\ell \neq t} y_{\ell}\rangle  \right]\\
 \leq \;& \frac{1}{n(n + 1)} \mathbb{E}\left[\sum_{t = 1}^{n + 1} \displaystyle{\max_{(\lambda,\gamma) \in \cal{C}}} \langle\gamma, \sum_{\ell \neq t}x_{\ell}\rangle  +  \langle\lambda, \sum_{\ell \neq t}y_{\ell}\rangle \right]  \\
 =\;& \frac{1}{(n + 1)} \sum_{t = 1}^{n + 1}\mathbb{E}\left[\frac{1}{n}\, \displaystyle{\max_{(\lambda,\gamma) \in \cal{C}}} \langle\gamma, \sum_{\ell \neq t}x_{\ell}\rangle  +  \langle\lambda, \sum_{\ell \neq t}y_{\ell}\rangle \right]  \\
 = \;&\mathbb{E}[\textrm{dist}(\hat{p}^{n},\hat{q}^{n})]. \nonumber
\end{aligned}
\end{equation}\\

The proof of the almost sure convergence follows directly from the facts that $\lim_{n\rightarrow \infty}  \hat{p}^n = p$, $\lim_{n\rightarrow \infty}  \hat{q}^n = q$, and the continuity of the distance function.
\end{proof}
 The above lemma suggests a natural upper-bound for the objective function in \eqref{eq:DistwithDExpExpanded}. More precisely, instead of solving \eqref{eq:DistwithDExpExpanded}, we can solve
 \begin{equation}
\label{eq:DistwithDExpExpanded2}
\min_{G} \;\; \mathbb{E} \left[
\begin{array}{ll}
 \displaystyle{\max_{(\lambda,\gamma)\in \mathcal{C}}} \;  &\displaystyle{\sum_{i=1}^m}\gamma_i \hat{P}_{D(G(z))}^i + \displaystyle{\sum_{j=1}^m} \lambda_j \hat{P}_{D(x)}^j\\
 \vspace{0.2cm}\\
\textrm{s.t.} \quad & \lambda_j + \gamma_i \leq c_{ij},\;\;\forall i,j
\end{array}
\right],
\end{equation}

\vspace{0.2cm}

\noindent where $\hat{P}_{D(G(z))}$ and $\hat{P}_{D(x)}$ are the unbiased estimators of ${P}_{D(G(z))}$ and ${P}_{D(x)}$ based on our i.i.d samples. Moreover, the expectation is taken with respect to both, the function $D$ as well as the batch of samples which is drawn for estimating the distributions. As we will see later, in practice it is easier to use the primal form for solving the inner problem in~\eqref{eq:DistwithDExpExpanded2}, i.e.,


\begin{equation}\nonumber
\min_{G} \;\; \mathbb{E} \left[
\begin{array}{ll}
\displaystyle{\min_{\pi \geq 0}} \;  &\displaystyle{\sum_{i=1}^m\sum_{j=1}^m} c_{ij} \pi_{ij}\\
\vspace{0.cm}\\
\textrm{s.t.} \; & \displaystyle{\sum_{j=1}^m} \pi_{ij} = \hat{P}_{D(G(z))}^i,\displaystyle{\sum_{i=1}^m} \pi_{ij} = \hat{P}_{D(x)}^j, \;\forall i,j
\end{array}
\right],
\end{equation}

\vspace{0.2cm}

To show the dependence of $c_{ij}$ to $G$, let us assume that our generator $G$ is generating the output $h(w,z)$ from the input $z$. Here $w$ represents the weights of the network needed to be learned. Moreover, in practice, the value  of $\hat{P}_{D(G(z))}^i$ is estimated by taking the average over all batch of data. Hence, by duplicating variables if necessary, we can re-write the above optimization problem as
\begin{equation}
\label{eq:min_min_gan}
\min_{w} \;\; \mathbb{E}_{z,x,D} \left[
\begin{array}{ll}
\displaystyle{\min_{\pi \geq 0}} \;  &\displaystyle{\sum_{i=1}^n\sum_{j=1}^n} \|D(h(w,z_{i})) - D(x_j)\|  \pi_{ij}\\
\vspace{0.cm}\\
\textrm{s.t.} \; & \pi \vec{1} = \frac{1}{n}, \; \; \pi^{T} \vec{1} = \frac{1}{n}
\end{array}
\right].
\end{equation}
Here, $n$ is the batch size and we ignored the entries of $\hat{P}_{D(G(z))}$ and $\hat{P}_{D(x)}$ that are zero. Notice that to obtain an algorithm with convergence guarantee for solving this optimization problem,  one can properly regularize the inner optimization problem to obtain unbiased estimates of the gradient of the objective function~\cite{nouiehed2019solving,sanjabi2018convergence}. However, in  this work, due to practical considerations, we suggest to \textit{approximately} solve the inner problem  and use the approximate solution for solving~\eqref{eq:min_min_gan}.

\noindent \textbf{Solving the inner-problem approximately.} In order to solve the inner problem in~\eqref{eq:min_min_gan}, we need to solve
\begin{equation}
\label{eq:OptimalAssignment}
\begin{split}
\displaystyle{\min_{\pi \geq 0}} \quad  &\displaystyle{\sum_{i=1}^n\sum_{j=1}^n} \|D(h(w,z_{i})) - D(x_j)\|  \pi_{ij}\\
\textrm{s.t.} \quad & \pi \vec{1} = \frac{1}{n}, \; \; \pi^{T} \vec{1} = \frac{1}{n}.  
\end{split}
\end{equation}

Notice that this problem is the classical \textit{optimal assignment problem} which can be solved using Hungarian method \cite{kuhn1955hungarian}, Auction algorithm \cite{bertsekas1988auction}, or many other methods proposed in the literature. Based on our observations, even the greedy method of assigning each column to the lowest unassigned row worked in our numerical experiments. The benefit of the greedy method is that it can be performed almost linearly in $m$ by the use of a proper hash function.   

Algorithm~\ref{alg:RP} summarizes our  proposed Generative Networks using Random Discriminator (GN-RD) algorithm  for solving~\eqref{eq:min_min_gan}.

\begin{algorithm} 
	\caption{Generative Networks using Random Discriminator (GN-RD)}
 	\label{alg:RP}
 	\SetKwInOut{Input}{Input}
 	\Input{ $w_{0}:$ Initialization for generator's parameter, $\alpha:$ Learning rate, $n:$ Batch size, $N_{Itr}$: Maximum iteration number }
	\For{$t =1:N_{\max} $}
	{           
	           Sample an i.i.d. batch of real data $ (x_1,\ldots,x_n)$ \\    
			   Sample an i.i.d. batch of noise $(z_1,\ldots,z_n)$\\
			   Create a random discriminator neural network $D$ with random weights \\
			   Solve \eqref{eq:OptimalAssignment} by finding the optimal assignment value between real data and generated sample \\
			   Update generator's parameter, $w_{t+1} = w_t - \alpha \nabla_{w} G(w_t) $

	}
    \SetKwInOut{Output}{Output}
	\Output{$G(w_{N_{\max}})$}

\end{algorithm}

\begin{remark}
The training approach in Algorithm~\ref{alg:RP} relies on two neural networks: the generative and the discriminator. Hence, Algorithm~\ref{alg:RP} can be viewed as a GANs training approach where we use a random discriminator  at each iteration of updating the generator. 
\end{remark}

\begin{remark}
The recent works \cite{li2018implicit,hoshen2018non} have  similarities  in terms of learning generative models through min-min formulation instead of min-max formulation. However, unlike their method, 1) our algorithm is based on mapping images via randomly generated discriminators; 2) In our analysis, we establish that this formulation leads to an upper-bound of the distance measure; 3) our algorithm is based on the use of optimal assignment, while the works \cite{li2018implicit,hoshen2018non} suggests a greedy matching, which is more difficult to understand and analyze.
\end{remark}

\section{Numerical Experiments}
In this section, we evaluate the performance of the proposed GN-RD algorithm for learning generative networks to create samples from MNIST ~\cite{lecun1998mnist} and Fashion-MNIST \cite{xiao2017fashion} datasets. As mentioned previously, the proposed algorithm does not require any optimization on the discriminator network and only needs randomly generated  discriminator to learn the underlying distribution of the data\footnote{All the experiments have been run on a machine with  single GeForce GTX 1050 Ti GPU.}. 

\subsection{Learning handwritten digits and fashion products}
In this section, we use GN-RD for generating samples from handwritten digits and Fashion-MNIST datasets. Each of these datasets contains 50K training samples. \\
\textbf{Architecture of the Neural Networks:}
The generator's Neural Network consists of two fully connected layer with 1024 and 6272 neurons. The output of the second fully connected layer is followed by two deconvolutional layers to generate the final $28\times 28$ image. 

The discriminator neural network has two convolution layers each followed by a max pool. The size of the both convolutional layers are 64. The last layer has been flatten to create the output. The design of both neural networks is summarized below: 
\begin{itemize}
  \item Generator: [FC(100, 1024), Leaky ReLU($\alpha = 0.2$),  FC(1024, 6272),  Leaky ReLU($\alpha = 0.2$), DECONV(64, kernel size = 4, stride = 2), Leaky ReLU(alpha = 0.2), DECONV(1, kernel size = 4, stride = 2), Sigmoid].
  \item Discriminator: [CONV(64, filter size = 5, stride = 1), Leaky ReLU(alpha = 0.2), Max Pool (kernel size = 2, stride = 2),  COVN(64, filter size = 5, stride = 1),  Max Pool (kernel size = 2, stride = 2), Flatting].
\end{itemize}
We have used originally proposed adversarial discriminator for  Wasserstein GAN  (WGAN) \cite{arjovsky2017wasserstein}, Wasserstein GAN with gradient penalty (WGAN-GP)\cite{gulrajani2017improved} \footnote{For WGAN and WGAN-GP implementation visit \url{https://github.com/igul222/improved_wgan_training}\label{code}} and Cram\'er GAN \cite{bellemare2017cramer}\footnote{For Cram\'er GAN implementation visit \url{https://github.com/jiamings/cramer-gan}}.

As mentioned in Algorithm \ref{alg:RP}, it is important to notice that unlike benchmark methods, the proposed method only optimizes the generator's parameters. However, at each iteration, weights in the convolutional layers of the discriminator are randomly generated from normal distribution.\\
\textbf{Hyper parameters:} We have used Adam with step size $0.001$ and $\beta_1 = 0.5$ and $\beta_2 = 0.9$ as the optimizer for our generator. The batch size is set to 100.

Fig.\ref{fig:MNSIT} shows the result of the generated digits and the corresponding inception score\cite{barratt2018note} using different benchmark methods. As seen from the figure, the proposed GN-RD is able to quickly learn the underlying distribution of the data and generate promising samples.

\begin{figure*}[ht]
\centering
\begin{tabular}{cccc}
\includegraphics[width=0.3\textwidth]{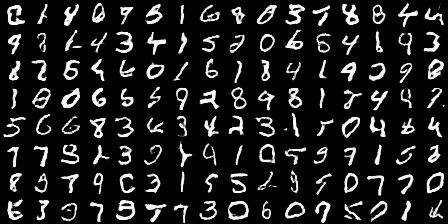} &
\includegraphics[width=0.3\textwidth]{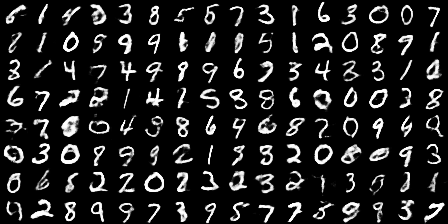} &
\includegraphics[width=0.3\textwidth]{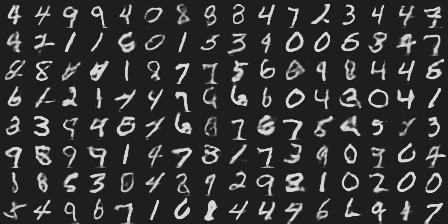} \\
\textbf{(a) WGAN}  & \textbf{(b)WGAN-GP} & \textbf{(c) Cramer GAN}  \\[6pt]
\end{tabular}
\begin{tabular}{cccc}
\includegraphics[width=0.3\textwidth]{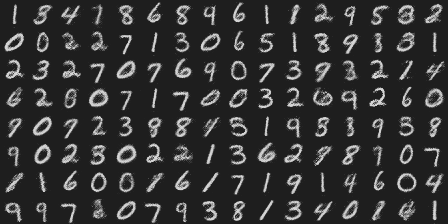} &
\includegraphics[width=0.3\textwidth, height = .2\textwidth ]{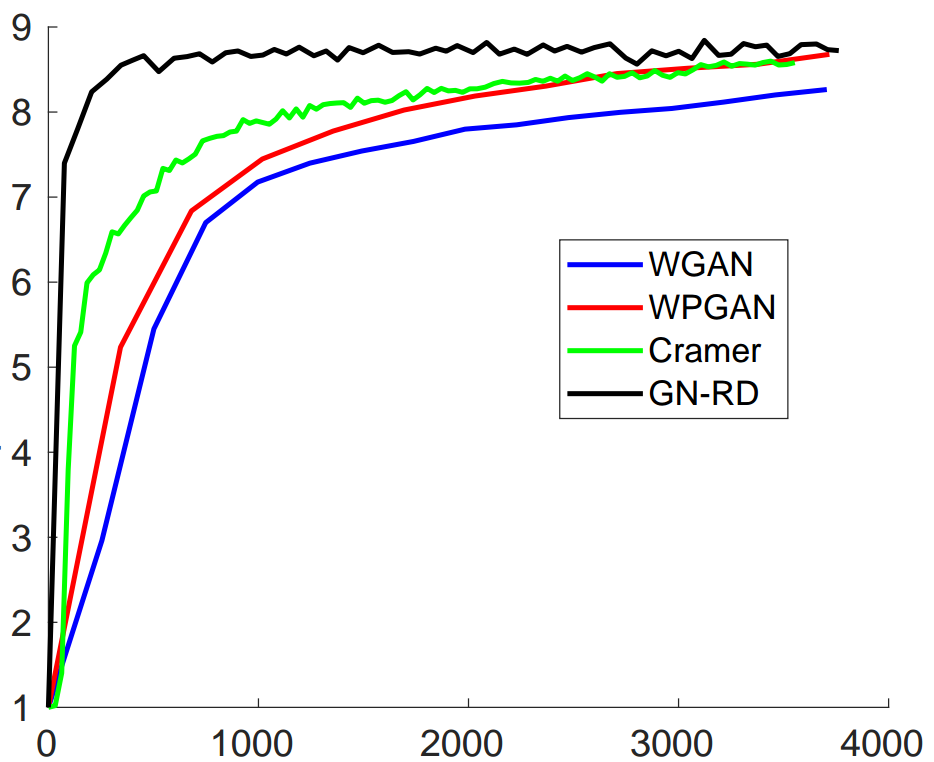} \\
\textbf{(d) GN-RD}  & \textbf{(e) Inception Score (\small{Over time in second })} \\[6pt]
\end{tabular}
\caption{Generating hand-written digits using MNIST dataset }
\label{fig:MNSIT}
 

\end{figure*}

Fig. \ref{fig:MNSIT-F} shows the result of using the proposed method for generating samples from fashion MNIST  dataset. The sample is generated only after 600 iterations ($\sim$ 10 minutes ) of the proposed method which shows that the GN-RD quickly converges and generates promising samples.

\begin{figure*}
\centering
\begin{tabular}{cccc}
\includegraphics[width=0.5\textwidth]{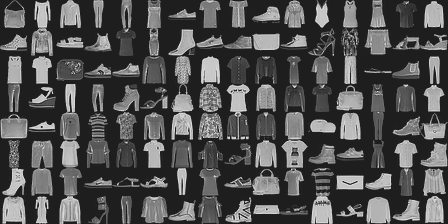} &
\includegraphics[width=0.5\textwidth]{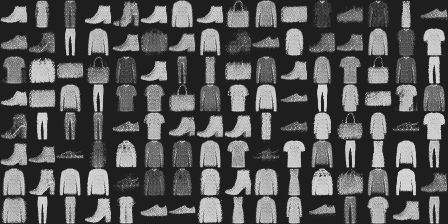} &
 \\
\textbf{(a) Original Data }  & \textbf{(b)GN-RD}  \\[6pt]
\end{tabular}
\caption{Generating fashion products using Fashion-MNIST dataset }
\label{fig:MNSIT-F}
 

\end{figure*}



\section{Conclusion}
Generative Adversarial Networks (GANs) have been able to learn the underlying distribution of the data and generate  samples from it. Training GANs is notoriously unstable  due to their non-convex min-max formulation. In this work, we propose the use of randomized discriminator to avoid facing the complexity of solving non-convex min-max problems. Evaluating the performance of the proposed method on real data set of MNIST and Fashion-MNIST shows the ability of the proposed method in generating promising samples without adversarial learning. 
\section*{Acknowledgement}
The  authors  would  like  to  thank  Mohammad Norouzi for his insightful feedback.
 
\bibliographystyle{IEEEtran}

\bibliography{references}

\end{document}